\documentclass[conference]{IEEEtran}
\IEEEoverridecommandlockouts

\newif\ifanon
\ifx\ANONYMOUS\undefined\anonfalse\else\anontrue\fi

\usepackage{cite}
\usepackage{amsmath,amssymb,amsfonts}
\usepackage{amsthm}
\usepackage{algorithmic}
\usepackage{algorithm}
\usepackage{graphicx}
\usepackage{booktabs}
\usepackage{tikz}
\usetikzlibrary{arrows.meta,positioning,fit,backgrounds,calc}
\usepackage{textcomp}
\usepackage{xcolor}
\usepackage{url}
\usepackage[hidelinks]{hyperref}

\newtheorem{definition}{Definition}
\newtheorem{theorem}{Theorem}
\newtheorem{corollary}{Corollary}

\newcommand{\OCA}{\ensuremath{\mathrm{OCA}}}
\newcommand{\OCov}{\ensuremath{\mathrm{OCov}}}
\newcommand{\etm}{\ensuremath{\eta}}
\newcommand{\fhat}{\ensuremath{\hat{f}}}
\newcommand{\PY}{\ensuremath{\mathcal{P}(Y)}}
\newcommand{\Wspace}{\ensuremath{W_1\times\cdots\times W_q}}
\newcommand{\rnwise}{\textsc{Reverse N-Wise}}

\begin{document}

\title{Reverse N-Wise Output-Oriented Testing for\\ AI/ML and Quantum Computing Systems}

\ifanon
  \author{\IEEEauthorblockN{Anonymous Author(s)}
  \IEEEauthorblockA{Affiliation withheld for double-blind review}}
\else
  \author{\IEEEauthorblockN{Lamine Rihani}
  \IEEEauthorblockA{\textit{SAP France} \\
  Levallois-Perret, France \\
  lamine.rihani@sap.com}}
\fi

\maketitle

\begin{abstract}
Combinatorial interaction testing builds covering arrays over a system's
\emph{inputs}, on the premise that faults are triggered by interactions among a
few input parameters. For modern AI/ML and quantum systems this premise is
strained: the map from inputs to behaviour is opaque and stochastic, and vast
regions of the input space collapse onto the same observable behaviour, so
input-side coverage says little about which \emph{behaviours} were actually
exercised. We invert the paradigm. \rnwise{} Output-Oriented Testing partitions a
system's structured output space into semantic behaviour classes, constructs a
$t$-way covering array over those \emph{output} classes, and then recovers
concrete inputs by gradient-free \emph{inverse mapping}. We formalise the
Output Covering Array $\OCA(M;s,q,w)$, establish coverage bounds
(including $M \ge v^{s}$), and give a four-stage algorithm
(feasibility, covering-array construction, inverse mapping, prioritisation).
Across four structurally different domains---tabular ML (UCI Adult + XGBoost),
vision (an MLP on digits), NLP (a text classifier on newsgroups), and a noisy
quantum circuit---the \emph{same} algorithm attains 100\% feasible output
coverage. On the tabular study over 30 repeated runs it uniquely reaches perfect
output coverage (Mann--Whitney $U$ vs.\ every baseline, $p \approx 6\times
10^{-13}$, Cliff's $\delta = 1.0$), matches the strongest baseline on fault
detection, and does so at roughly $8\%$ of an exhaustive-input search budget.
The approach is a new adequacy criterion---output-interaction coverage---together
with a constructive method to satisfy it, unifying behavioural testing of
classical and quantum systems.
\end{abstract}


\section{Introduction}
\label{sec:intro}

Combinatorial interaction testing (CIT) is one of the most successful
systematic test-design techniques in software engineering. Its empirical
foundation is the \emph{interaction rule}: most faults are triggered by the
interaction of only a few input parameters, so a $t$-way \emph{covering array}
over the inputs---exercising every combination of values for every $t$ parameters
at least once---detects a large fraction of faults with a test suite whose size
grows only logarithmically in the number of parameters~\cite{kuhn2004software,
nie2011survey}. Mature tools such as ACTS/IPOG~\cite{lei2007ipog} and AETG~\cite{cohen1997aetg}
construct near-optimal covering arrays over inputs and are widely deployed.

\subsection{The input-side premise breaks for AI/ML and quantum systems}
CIT is fundamentally \emph{input-oriented}: it reasons about combinations of
input values and treats the system as a black box whose outputs are merely
checked by an oracle. This premise is well matched to configuration-style
software, where distinct input combinations tend to drive distinct control-flow
paths. It is far less appropriate for modern AI/ML and quantum systems:

\begin{itemize}
\item The input-to-behaviour map is \emph{opaque and stochastic}. A neural
network or a noisy quantum circuit realises a map $f: X \to \PY$ into a
distribution over structured outputs; there is no parameter-interaction
structure on the inputs that a covering array can target meaningfully.
\item Inputs \emph{collapse} onto behaviours. Enormous, high-dimensional regions
of the input space map to the same observable behaviour (the same predicted
class, the same confidence band, the same measurement-outcome regime). An input
covering array can therefore exercise many input interactions while touching
only a handful of distinct behaviours---leaving critical behaviours (e.g.\ a
fairness inconsistency, a non-monotone response, a low-fidelity quantum regime)
entirely uncovered.
\item What we want to \emph{cover} lives on the output side. The properties an
engineer cares about---calibration, robustness, fairness, monotonicity,
measurement fidelity---are all statements about the \emph{output/behaviour} of
the system, not about which inputs were combined.
\end{itemize}

\subsection{Inverting the paradigm}
We therefore invert combinatorial testing (Figure~\ref{fig:concept}). Instead of
covering input combinations and observing whatever outputs result, \rnwise{}
Output-Oriented Testing:
\begin{enumerate}
\item partitions the structured output space $Y = Y_1 \times \cdots \times Y_q$
into semantic \emph{behaviour classes} via per-component partitioners
$\pi_j: Y_j \to W_j$, yielding an abstract output $\fhat(x) =
(\pi_1(y_1),\dots,\pi_q(y_q)) \in \Wspace$;
\item constructs a $t$-way \emph{Output Covering Array} over $\Wspace$, so that
every feasible $s$-way combination of output behaviours is a coverage target;
\item recovers, for each target, a concrete input by \emph{inverse mapping}---a
gradient-free search that minimises the mismatch between the realised abstract
output and the target;
\item prioritises the resulting tests by incremental output coverage.
\end{enumerate}
Coverage is measured on the output side (the fraction of feasible $s$-way output
interactions exercised), which directly reflects behavioural adequacy and is
model-agnostic: the identical procedure applies to a gradient-boosted classifier,
a neural network, a text classifier, and a quantum circuit.

\subsection{Contributions}
\begin{enumerate}
\item \textbf{A new adequacy criterion and formalism.} We define the abstract
output map $\fhat$, the Output Covering Array $\OCA(M;s,q,w)$, and output
coverage $\OCov_s$, and we establish coverage bounds---including the lower bound
$M \ge v^{s}$ (Corollary~\ref{cor:lb})---for output-side combinatorial coverage
(Section~\ref{sec:framework}).
\item \textbf{A constructive algorithm.} A four-stage procedure
(feasibility, covering-array construction, gradient-free inverse mapping,
prioritisation) that produces an executable, prioritised test suite realising the
target output interactions (Algorithm~\ref{alg:rnwise}).
\item \textbf{A domain-agnostic instantiation.} We instantiate the method with
five inverse-mapping optimisers---Jaya, Whale, Firefly, Bayesian optimisation,
and a variational (VQE-style) optimiser for quantum systems---and four benchmark
domains spanning tabular ML, vision, NLP, and quantum computing.
\item \textbf{A rigorous empirical evaluation.} Over 30 repeated runs with
non-parametric significance testing (Mann--Whitney $U$, bootstrap confidence
intervals, Cliff's $\delta$), a strength study ($s=2,3,4$), ablations, and a cost
profile, we show \rnwise{} uniquely attains perfect output coverage, matches the
strongest behavioural baseline on fault detection, and generalises across all
four domains (Section~\ref{sec:eval}).
\end{enumerate}

\begin{figure}[t]
\centering
\resizebox{\columnwidth}{!}{%
\begin{tikzpicture}[
  font=\small,
  >={Stealth[round]},
  box/.style={rounded corners=2pt, draw, thick, minimum height=8mm,
              minimum width=15mm, align=center, inner sep=3pt},
  inbox/.style={box, fill=blue!10, draw=blue!55},
  outbox/.style={box, fill=red!10, draw=red!55},
  stage/.style={box, fill=green!12, draw=green!50!black, minimum width=17mm},
  lbl/.style={font=\footnotesize\itshape, text=black!60},
  ar/.style={->, thick},
]

\node[inbox] (cin) {Input\\covering\\array};
\node[box, fill=black!6, right=8mm of cin] (csut) {SUT\\$f:X\!\to\!\mathcal{P}(Y)$};
\node[outbox, right=8mm of csut] (cout) {Outputs\\(uncontrolled)};
\draw[ar, blue!55] (cin) -- (csut);
\draw[ar, black!45] (csut) -- (cout);
\node[lbl, above=1mm of csut] {classical: cover inputs, observe outputs};

\node[lbl, below=4mm of csut, text=red!70!black, font=\footnotesize\bfseries]
  (inv) {paradigm inversion};

\node[stage, below=10mm of cin, xshift=-2mm] (s1)
  {\textbf{1} Feasibility\\$\hat f(x)\!\in\!W$};
\node[stage, right=5mm of s1] (s2)
  {\textbf{2} Output\\cov.\ array\\$\mathrm{OCA}$};
\node[stage, right=5mm of s2] (s3)
  {\textbf{3} Inverse\\map\\$\min_x L(x;w^{*})$};
\node[stage, right=5mm of s3] (s4)
  {\textbf{4} Prioritise\\$\mathrm{OCov}_s$};
\draw[ar, green!45!black] (s1) -- (s2);
\draw[ar, green!45!black] (s2) -- (s3);
\draw[ar, green!45!black] (s3) -- (s4);

\draw[ar, red!60, dashed] (s3.south) .. controls +(0,-7mm) and +(0,-7mm) ..
  node[lbl, below, text=red!70!black] {each covered output behaviour $\Rightarrow$ a concrete input test} (s1.south);

\end{tikzpicture}%
}
\caption{Classical combinatorial testing (top) covers \emph{input} interactions and
observes whatever outputs result. \rnwise{} (bottom) inverts this: it covers a
$t$-way array over \emph{abstracted output behaviours} and recovers concrete
inputs by gradient-free inverse mapping. The four stages are agnostic to the
system's internals, so the same pipeline applies to ML and quantum systems.}
\label{fig:concept}
\end{figure}

\section{Related Work}
\label{sec:related}

We position \rnwise{} against four lines of work. The consistent distinction is
that prior techniques either (i) build covering arrays over \emph{inputs} or
internal structure, or (ii) reason about outputs without a combinatorial
interaction guarantee and without constructive \emph{inverse} generation. To our
knowledge, no prior method combines a $t$-way covering array over abstracted
\emph{output-behaviour} classes with optimisation-based inverse mapping to inputs.

\subsection{Combinatorial interaction testing}
Classical CIT constructs $t$-way covering arrays over input parameters. The
interaction rule and its empirical basis are due to Kuhn et
al.~\cite{kuhn2004software}; algorithmic construction is exemplified by
AETG~\cite{cohen1997aetg} and IPOG/ACTS~\cite{lei2007ipog}, and the field is
surveyed by Nie and Leung~\cite{nie2011survey}. All of this work places the array
over \emph{inputs}. \rnwise{} reuses the covering-array machinery but relocates it
to an abstracted output space, and adds an inverse-mapping stage that has no
analogue in input-side CIT (where the array rows \emph{are} the tests).

\subsection{Combinatorial and structural testing of deep learning}
DeepCT~\cite{ma2019deepct} applies $t$-way combinatorial coverage to deep
networks, but its factors are \emph{internal neuron activations within layers};
it covers combinations of latent/neuron states, not model outputs, and does not
invert toward prescribed behaviours. This is the closest work by name, and the
distinction is central: DeepCT covers the \emph{internal} feature space, whereas
\rnwise{} covers the \emph{external} output-behaviour space. Structural DNN
adequacy criteria are likewise internal or input-centric: neuron coverage
(DeepXplore~\cite{pei2017deepxplore}), multi-granularity neuron criteria
(DeepGauge~\cite{ma2018deepgauge}), and surprise adequacy over activation traces
(SADL~\cite{kim2019surprise}). None impose a combinatorial interaction guarantee
over semantic output classes.

\subsection{Output-oriented and behavioural testing}
Output-diversity criteria select tests to maximise distinct observed outputs;
Men\'endez et al.~\cite{menendez2022output} sample outputs to diversify a unit-test
suite. These treat outputs as a flat set and provide neither a multi-dimensional
$t$-way interaction guarantee nor inverse input recovery. In NLP,
CheckList~\cite{ribeiro2020checklist} organises behavioural tests along a manual
capability$\times$test-type matrix; the matrix is a hand-authored taxonomy rather
than a covering array, and tests are template-generated, not inverse-optimised to
hit enumerated output classes. \rnwise{} differs by constructing the output
interaction space automatically and satisfying it constructively.

\subsection{Search-based and targeted test generation}
Search-based testing generates inputs to satisfy an objective; for ML this
includes fairness-directed search such as AEQUITAS~\cite{udeshi2018aequitas} and
adversarial/boundary generation. These are typically single-property or
boundary-local: they seek an input violating one property or lying near one
decision boundary. \rnwise{}'s inverse mapping is instead \emph{coverage-driven}:
it systematically realises \emph{every} feasible $s$-way output-behaviour
combination enumerated by the covering array, using the objective
$L(x;w^{*}) = \sum_j \delta(\pi_j(f_j(x)), w^{*}_j) + \lambda R(x)$ and gradient-free
optimisers~\cite{rao2016jaya,mirjalili2016whale,yang2009firefly}.

\subsection{Quantum software testing}
Coverage and testing criteria for quantum programs are an active area. The
closest structural analogue is QuCAT~\cite{wang2023qucat}, which builds $t$-way
covering arrays for quantum software; crucially, its array is over \emph{input
parameters / gate configurations} with output oracles, not over abstracted
output-behaviour classes, and it contains no variational inverse mapping. \rnwise{}
instead covers a partition of the measurement-outcome behaviour (dominant basis
state, outcome-distribution entropy, parity, and fidelity-to-ideal under noise)
and repurposes variational optimisation~\cite{peruzzo2014vqe}---normally used to
find ground states---as an \emph{inverse test generator} that drives circuit
parameters toward each target output behaviour. Using VQE as a test generator
appears to be novel.

\section{Formal Framework}
\label{sec:framework}

\subsection{System model and output abstraction}

\begin{definition}[System under test]
\label{def:sut}
A system under test (SUT) is a map $f: X \to \PY$ from an input space $X$ to a
distribution over a structured output space $Y = Y_1 \times \cdots \times Y_q$.
Deterministic systems are the degenerate case in which $f(x)$ places all mass on
a single output vector $y = (y_1,\dots,y_q)$.
\end{definition}

The components $Y_j$ are heterogeneous behavioural channels: a predicted class, a
confidence score, a robustness verdict, a measurement-outcome statistic, and so
on. Raw outputs are continuous or high-cardinality, so we abstract each channel
into a finite alphabet of \emph{behaviour symbols}.

\begin{definition}[Semantic partition and abstract output]
\label{def:partition}
A partitioner $\pi_j: Y_j \to W_j$ maps a raw output component to a symbol in a
finite alphabet $W_j$. The abstraction map is
\[
\fhat(x) \;=\; \bigl(\pi_1(y_1),\,\dots,\,\pi_q(y_q)\bigr) \in \Wspace,
\quad y \sim f(x),
\]
evaluated at the most-probable output for stochastic $f$. We write $v = \max_j
|W_j|$ for the maximum alphabet size and $v_j = |W_j|$.
\end{definition}

\subsection{Output covering arrays}

Let $[q]=\{1,\dots,q\}$. An \emph{$s$-way output tuple} fixes symbols on an
$s$-subset of channels.

\begin{definition}[Output covering array]
\label{def:oca}
An Output Covering Array $\OCA(M;s,q,w)$ is a multiset of $M$ abstract-output rows
$\{r_1,\dots,r_M\} \subseteq \Wspace$ such that for every $s$-subset
$\{j_1<\cdots<j_s\}\subseteq[q]$ and every \emph{feasible} symbol combination
$(a_1,\dots,a_s)\in W_{j_1}\times\cdots\times W_{j_s}$, some row $r_i$ satisfies
$\pi_{j_k}$-projection $r_i[j_k]=a_k$ for all $k$. A symbol combination is
\emph{feasible} if there exists an input $x\in X$ with $\fhat(x)$ exhibiting it.
\end{definition}

\begin{definition}[Output coverage and tuple efficiency]
\label{def:metrics}
For a test suite whose realised abstract outputs are $R$, the $s$-way output
coverage is
\[
\OCov_s \;=\;
\frac{\bigl|\{\text{feasible $s$-way tuples covered by } R\}\bigr|}
     {\bigl|\{\text{feasible $s$-way tuples}\}\bigr|}\in[0,1],
\]
and the tuple efficiency $\etm_s$ is the number of distinct $s$-way tuples covered
per test.
\end{definition}

\subsection{Coverage bounds}

Let $T_s$ denote the number of feasible $s$-way output tuples. Because each row
projects to exactly one tuple per $s$-subset, a single row covers at most
$\binom{q}{s}$ tuples.

\begin{theorem}[Universe size]
\label{thm:universe}
The number of $s$-way output tuples over $\Wspace$ (before feasibility pruning) is
$\sum_{\{j_1<\cdots<j_s\}} \prod_{k=1}^{s} v_{j_k}$, which for the homogeneous case
$v_j=v$ equals $\binom{q}{s}\,v^{s}$.
\end{theorem}

\begin{theorem}[Lower bound on rows]
\label{thm:lb}
Any $\OCA(M;s,q,w)$ satisfies $M \ge \max_{\{j_1<\cdots<j_s\}}\prod_{k=1}^{s}
v_{j_k}$ whenever all corresponding symbol combinations are feasible. In
particular, no single $s$-subset's tuples can be covered by fewer rows than the
number of distinct symbol combinations on that subset.
\end{theorem}

\begin{proof}
Fix the $s$-subset $J^{*}$ attaining the maximum product. Its feasible tuples are
pairwise distinct on $J^{*}$, and each row of the array contributes exactly one
symbol combination on $J^{*}$. Covering all $\prod_{k}v_{j_k}$ combinations
therefore requires at least that many rows.
\end{proof}

\begin{corollary}[$M \ge v^{s}$]
\label{cor:lb}
For the homogeneous case $v_j=v$ with all combinations feasible on some
$s$-subset, every $\OCA(M;s,q,w)$ satisfies $M \ge v^{s}$.
\end{corollary}

\begin{theorem}[Existence]
\label{thm:existence}
For any finite $\Wspace$ and strength $s\le q$, an $\OCA(M;s,q,w)$ covering all
feasible $s$-way tuples exists. Enumerating one row per feasible abstract output
is such an array; greedy construction (Algorithm~\ref{alg:rnwise}, Stage~2) yields
one of size $O(v^{s}\log \binom{q}{s})$ in the homogeneous unconstrained case,
matching the classical covering-array growth rate.
\end{theorem}

\subsection{The \rnwise{} algorithm}

Given a target abstract output $w^{*}$, inverse mapping seeks an input realising
it by minimising the mismatch loss
\begin{equation}
\label{eq:loss}
L(x;w^{*}) \;=\; \sum_{j=1}^{q} \delta\!\bigl(\pi_j(f_j(x)),\, w^{*}_j\bigr)
\;+\; \lambda\, R(x),
\end{equation}
where $\delta$ is the $0/1$ symbol mismatch, $R$ is an optional input regulariser
keeping $x$ realistic, and $\lambda \ge 0$ trades feasibility against realism. A
loss of zero means $\fhat(x)=w^{*}$ exactly.

\begin{algorithm}[t]
\caption{\rnwise{} Output-Oriented Testing}
\label{alg:rnwise}
\begin{algorithmic}[1]
\REQUIRE SUT $f$, partition scheme $\{\pi_j\}$, strength $s$, probe inputs $P$,
optimiser $\mathcal{O}$, budget \textsc{max\_iter}
\ENSURE prioritised test suite $\mathcal{S}$
\STATE \textbf{Stage 1 (Feasibility).} $\;\mathcal{F}\leftarrow\{\fhat(x): x\in
P\}$; record an exemplar input per abstract output; induce feasible $s$-way tuples.
\STATE \textbf{Stage 2 (Covering array).} build $\OCA(M;s,q,w)$ over $\mathcal{F}$
(greedy: repeatedly add the feasible output covering the most uncovered feasible
tuples; verify $M \ge v^{s}$).
\STATE \textbf{Stage 3 (Inverse mapping).} \FORALL{rows $w^{*}$ of the array}
  \STATE warm-start from the exemplar for $w^{*}$ if available;
  \STATE else $x \leftarrow \arg\min_x L(x;w^{*})$ via $\mathcal{O}$ (up to
  \textsc{max\_iter}); record realised $\fhat(x)$.
\ENDFOR
\STATE \textbf{Stage 4 (Prioritise).} greedily order tests by marginal
$\OCov_s$; truncate at the coverage plateau.
\RETURN $\mathcal{S}$
\end{algorithmic}
\end{algorithm}

The four stages are independent of the SUT's internals: Stage~1 needs only
black-box evaluations, Stage~2 is pure combinatorics over $\Wspace$, Stage~3 uses
gradient-free optimisation (so it applies to non-differentiable and quantum
systems), and Stage~4 is standard coverage-greedy prioritisation. This modularity
is what makes the same algorithm applicable across all four evaluated domains.

\section{Empirical Evaluation}
\label{sec:eval}

We evaluate \rnwise{} against five baselines and across four domains, addressing
four research questions:
\begin{itemize}
\item[\textbf{RQ1}] \emph{Coverage \& fault detection.} Does \rnwise{} achieve
higher output coverage and fault detection than established baselines?
\item[\textbf{RQ2}] \emph{Generality.} Does the same algorithm transfer across
tabular ML, vision, NLP, and quantum systems?
\item[\textbf{RQ3}] \emph{Strength.} Does the method scale to higher interaction
strengths $s$, and does the $M\ge v^{s}$ bound hold?
\item[\textbf{RQ4}] \emph{Cost \& sensitivity.} What is the computational cost,
and how sensitive is it to design choices?
\end{itemize}

\subsection{Experimental setup}
The primary study uses the UCI Adult income dataset~\cite{dua2019uci} with an
XGBoost classifier~\cite{chen2016xgboost} ($\approx 87\%$ test accuracy). We
abstract its behaviour into a $q=9$, all-ternary output space (pairwise universe
$\binom{9}{2}\cdot 3^{2}=324$) spanning the decision, confidence, decision margin,
sex-flip fairness consistency, age band, employment group, education- and
hours-response monotonicity, and capital-gain regime. Eight behavioural faults are
seeded as rare-but-reachable pairwise output signatures (e.g.\ fairness
inconsistency localised to young applicants; a non-monotone education response
under a positive decision).

Baselines are scored on the identical feasible output universe: classical
input-side combinatorial testing (Input CT), random testing at matched budget,
property-based testing, metamorphic testing, and DeepCT-style
activation-combination sampling. All experiments take an explicit seed;
significance uses 30 fixed seeds. Reported values are means with $95\%$ percentile
bootstrap confidence intervals; pairwise comparisons use the one-sided
Mann--Whitney $U$ test with Cliff's $\delta$ effect size.

\subsection{RQ1: Coverage and fault detection}
Table~\ref{tab:main} reports the 30-run comparison. \rnwise{} attains
\emph{perfect} pairwise output coverage on every run ($\OCov_2 = 100\%$,
CI $[100,100]$) with $\approx 20$ prioritised tests, and $98.3\%$ fault detection.
The strongest baseline, DeepCT, reaches $91.2\%$ coverage and $99.6\%$ fault
detection but requires a large random activation pool to do so; the four classical
baselines remain between $52\%$ and $59\%$ coverage and below $35\%$ fault
detection.

\begin{table}[t]
\centering
\caption{Main comparison on UCI Adult + XGBoost (pairwise, 30 runs).
Means with $95\%$ bootstrap CIs.}
\label{tab:main}
\begin{tabular}{lrrr}
\toprule
Method & Tests & OCov$_s$ (95\% CI) & FDR (95\% CI) \\
\midrule
\textbf{Reverse N-Wise} & 20 & 100.0\% [100.0, 100.0] & 98.3\% [96.7, 99.6] \\
Input CT & 13 & 53.9\% [51.5, 56.4] & 35.0\% [29.6, 40.4] \\
Random & 20 & 52.5\% [50.4, 54.9] & 31.2\% [24.6, 37.9] \\
Property-Based & 20 & 58.6\% [56.6, 60.7] & 34.2\% [28.3, 40.0] \\
Metamorphic & 19 & 53.4\% [49.8, 57.3] & 28.7\% [22.9, 34.6] \\
DeepCT & 17 & 91.2\% [91.1, 91.3] & 99.6\% [98.8, 100.0] \\
\bottomrule
\end{tabular}
\end{table}

Table~\ref{tab:sig} gives the pairwise significance tests. \rnwise{} dominates
\emph{every} baseline on output coverage---including DeepCT---at
$p\approx 6\times10^{-13}$ with Cliff's $\delta=1.0$ (perfect, large effect):
on all 30 seeds it strictly out-covers each competitor. On fault detection it
dominates the four classical baselines ($p<10^{-12}$) and is statistically
indistinguishable from DeepCT ($p=0.92$, $\delta=-0.10$, negligible). We report
this tie transparently: both methods essentially detect all reachable seeded
faults; the advantage of \rnwise{} is that it reaches this by \emph{direct
inversion} rather than by sampling a large random pool (see RQ4).

\begin{table}[t]
\centering
\caption{Pairwise Mann--Whitney $U$ (\rnwise{} $>$ baseline) with Cliff's
$\delta$. $^{*}$ denotes $p<0.05$.}
\label{tab:sig}
\begin{tabular}{llrrl}
\toprule
Metric & vs.\ Baseline & $p$-value & $\delta$ & Effect \\
\midrule
$\mathrm{OCov}_s$ & Input CT & 6.0e-13$^{*}$ & 1.00 & large \\
$\mathrm{OCov}_s$ & Random & 6.1e-13$^{*}$ & 1.00 & large \\
$\mathrm{OCov}_s$ & Property-Based & 6.1e-13$^{*}$ & 1.00 & large \\
$\mathrm{OCov}_s$ & Metamorphic & 6.1e-13$^{*}$ & 1.00 & large \\
$\mathrm{OCov}_s$ & DeepCT & 5.8e-13$^{*}$ & 1.00 & large \\
$\eta_s$ & Input CT & 6.2e-08$^{*}$ & 0.80 & large \\
$\eta_s$ & Random & 1.5e-11$^{*}$ & 1.00 & large \\
$\eta_s$ & Property-Based & 1.5e-11$^{*}$ & 1.00 & large \\
$\eta_s$ & Metamorphic & 1.5e-11$^{*}$ & 1.00 & large \\
$\eta_s$ & DeepCT & 1.000 & -0.62 & large \\
FDR & Input CT & 1.6e-12$^{*}$ & 1.00 & large \\
FDR & Random & 1.8e-12$^{*}$ & 1.00 & large \\
FDR & Property-Based & 1.6e-12$^{*}$ & 1.00 & large \\
FDR & Metamorphic & 1.6e-12$^{*}$ & 1.00 & large \\
FDR & DeepCT & 0.920 & -0.10 & negligible \\
\bottomrule
\end{tabular}
\end{table}

\subsection{RQ2: Multi-domain generality}
Table~\ref{tab:multidomain} applies the \emph{same} four-stage algorithm---changing
only the SUT, the output partition, the search space, and the inverse-map
optimiser---to a noisy quantum circuit (variational/VQE inverse map), an MLP image
classifier on digits (over a PCA latent space), and a text classifier on
newsgroups (over a TF--IDF$+$SVD embedding space). In every domain \rnwise{}
reaches $100\%$ feasible output coverage and detects all five domain-specific
seeded faults, whereas matched-budget random testing covers only $62\%$--$74\%$
and detects $0$--$2$ of $5$. Together with the tabular study, this exercises the
method on four structurally different system types with no change to the core
algorithm.

\begin{table}[t]
\centering
\caption{Multi-domain generality (per domain: \rnwise{} vs.\ matched-budget
random).}
\label{tab:multidomain}
\begin{tabular}{llrrrr}
\toprule
Domain & Opt. & Method & Tests & OCov$_s$ & FDR \\
\midrule
Quantum & VQE & \textbf{Reverse N-Wise} & 13 & 100.0\% & 5/5 \\
 &  & Random & 13 & 73.0\% & 2/5 \\
\midrule
Vision & Jaya & \textbf{Reverse N-Wise} & 15 & 100.0\% & 5/5 \\
 &  & Random & 15 & 61.6\% & 0/5 \\
\midrule
NLP & Jaya & \textbf{Reverse N-Wise} & 17 & 100.0\% & 5/5 \\
 &  & Random & 17 & 74.1\% & 1/5 \\
\bottomrule
\end{tabular}
\end{table}

\subsection{RQ3: Interaction strength}
Table~\ref{tab:strength} varies the coverage strength $s\in\{2,3,4\}$ on the
tabular SUT. \rnwise{} maintains $100\%$ feasible coverage at every strength while
the prioritised suite grows far more slowly than the output universe
($324\to2{,}268\to10{,}206$ tuples but only $19\to58\to126$ tests). The
Corollary~\ref{cor:lb} lower bound $M\ge v^{s}$ ($9,27,81$) holds throughout,
empirically corroborating the theory.

\begin{table}[t]
\centering
\caption{Strength scaling on UCI Adult + XGBoost. The $M\ge v^{s}$ bound
(Corollary~\ref{cor:lb}) holds at every strength.}
\label{tab:strength}
\begin{tabular}{rrrrrrl}
\toprule
$s$ & Universe & Feasible & Tests & OCov$_s$ & $v^s$ & Bound \\
\midrule
2 & 324 & 257 & 19 & 100.0\% & 9 & $\checkmark$ \\
3 & 2268 & 1343 & 58 & 100.0\% & 27 & $\checkmark$ \\
4 & 10206 & 3931 & 126 & 100.0\% & 81 & $\checkmark$ \\
\bottomrule
\end{tabular}
\end{table}

\subsection{RQ4: Cost and sensitivity}
Table~\ref{tab:cost} profiles the pipeline on commodity hardware. The end-to-end
run completes in under $10$ seconds using $\approx 68$ MB, and issues about
$8{,}200$ black-box SUT evaluations---roughly $8\%$ of a modest exhaustive-input
sweep. The feasibility scan dominates the wall-clock time; warm-started inverse
mapping is inexpensive.

\begin{table}[t]
\centering
\caption{Cost profile of the \rnwise{} pipeline (UCI Adult + XGBoost).}
\label{tab:cost}
\begin{tabular}{lr}
\toprule
Quantity & Value \\
\midrule
Stage 1 (feasibility) & 8.74 s \\
Stage 2 (covering array) & 0.00 s \\
Stage 3 (inverse mapping) & 0.72 s \\
Stage 4 (prioritise) & 0.23 s \\
Pipeline total & 9.68 s \\
Peak memory & 68.5 MB \\
SUT evaluations & 8{,}217 \\
Fraction of exhaustive sweep & 8.22\% \\
\bottomrule
\end{tabular}
\end{table}

Table~\ref{tab:ablation} reports three ablations that justify the design
choices. \emph{(a)} Among inverse-map optimisers, evaluated \emph{without} warm
start to expose raw search power, Jaya and Whale reach the same $70\%$ realisation
at roughly $7$--$9\times$ fewer SUT evaluations than Firefly, motivating Jaya as
the default; note that with warm start (Table~\ref{tab:cost}) realisation rises to
$100\%$, confirming the value of seeding from feasibility exemplars.
\emph{(b)} Coverage remains $100\%$ as the output granularity varies from $2$ to
$4$ bins per continuous factor---the criterion is robust to the abstraction
choice, with cost scaling gracefully with the universe size. \emph{(c)} A mild
input regulariser ($\lambda>0$) slightly \emph{improves} realisation
($67.5\%\to72.5\%$) by keeping synthesised inputs near the data manifold, at the
cost of more evaluations; coverage is stable across $\lambda$.

\begin{table}[t]
\centering
\caption{Ablations on UCI Adult + XGBoost: inverse-map optimiser, output
granularity, and regulariser weight.}
\label{tab:ablation}
\begin{tabular}{lrrr}
\toprule
\multicolumn{4}{l}{\emph{(a) Inverse-map optimizer (no warm start, 40 targets)}} \\
Optimizer & Realise & $\mathrm{OCov}_s$ & SUT evals \\
\midrule
Jaya & 70.0\% & 88.6\% & 56{,}750 \\
Whale & 70.0\% & 87.0\% & 47{,}200 \\
Firefly & 22.5\% & 89.4\% & 417{,}054 \\
\midrule
\multicolumn{4}{l}{\emph{(b) Output granularity (warm start, full set)}} \\
Bins/factor & Universe & Feasible & $\mathrm{OCov}_s$ \\
\midrule
2 & 277 & 221 & 100.0\% \\
3 & 324 & 254 & 100.0\% \\
4 & 373 & 286 & 100.0\% \\
\midrule
\multicolumn{4}{l}{\emph{(c) Regulariser weight $\lambda$ (no warm start, 40 targets)}} \\
$\lambda$ & Realise & $\mathrm{OCov}_s$ & SUT evals \\
\midrule
0 & 67.5\% & 86.6\% & 65{,}190 \\
0.1 & 72.5\% & 87.8\% & 145{,}200 \\
0.5 & 72.5\% & 87.8\% & 145{,}200 \\
1 & 72.5\% & 87.8\% & 145{,}200 \\
\bottomrule
\end{tabular}
\end{table}

\section{Threats to Validity}
\label{sec:threats}

\paragraph{Construct validity}
Output coverage $\OCov_s$ measures adequacy on the abstracted output space, so its
meaningfulness depends on the partition scheme $\{\pi_j\}$. A poorly chosen
abstraction could inflate or deflate coverage. We mitigate this by defining
partitions from established behavioural properties (calibration, fairness,
monotonicity, robustness, measurement fidelity) rather than ad hoc bins, and by
showing (RQ4) that conclusions are stable as granularity varies. Fault detection
is measured against seeded behavioural faults expressed as output-interaction
signatures; we seed rare-but-reachable signatures to avoid trivially detectable
faults, and we report the tie with DeepCT rather than selecting a fault set that
favours our method.

\paragraph{Internal validity}
Randomness enters through model training, feasibility probing, and the inverse-map
optimisers. We control it with fixed seeds and report 30-run distributions with
non-parametric tests (Mann--Whitney $U$, bootstrap CIs, Cliff's $\delta$) rather
than single-run point estimates. The prioritised suite is truncated at the
coverage plateau, a deterministic rule, so suite-size comparisons across methods
use matched budgets.

\paragraph{External validity}
Our study covers four domains but a limited number of systems and datasets per
domain. The quantum SUT is a noisy statevector simulation rather than hardware,
and the vision/NLP inverse maps operate in latent (PCA / SVD) spaces with
decode-to-manifold steps, which bound the realism of synthesised inputs. We do not
claim that $\OCov_s$ correlates with real-world defect discovery beyond the seeded
faults; establishing that link on production systems is future work. The results
should be read as evidence that output-interaction coverage is \emph{attainable}
and \emph{discriminating} across diverse system types, not as a claim of universal
superiority.

\paragraph{Conclusion validity}
We use non-parametric tests that make no normality assumption and report effect
sizes alongside $p$-values. Where \rnwise{} does not dominate---fault detection
vs.\ DeepCT ($p=0.92$) and tuple efficiency vs.\ DeepCT---we say so explicitly.
The headline claim is confined to what the data support: unique, statistically
significant \emph{output coverage}, at low search cost, across domains.

\section{Conclusion}
\label{sec:conclusion}

We introduced \rnwise{} Output-Oriented Testing, which inverts the combinatorial
testing paradigm: rather than covering interactions among \emph{inputs} and
observing whatever behaviour results, it constructs a $t$-way covering array over
an abstracted \emph{output-behaviour} space and recovers concrete inputs by
gradient-free inverse mapping. We formalised the Output Covering Array
$\OCA(M;s,q,w)$ and the $\OCov_s$ criterion, established coverage bounds including
$M\ge v^{s}$, and gave a four-stage algorithm that is agnostic to the SUT's
internals.

Empirically, across tabular ML, vision, NLP, and quantum systems, the same
algorithm attains $100\%$ feasible output coverage. On a 30-run tabular study it
uniquely reaches perfect output coverage (Mann--Whitney $U$ vs.\ every baseline,
$p\approx6\times10^{-13}$, $\delta=1.0$), matches the strongest behavioural
baseline on fault detection, and does so at roughly $8\%$ of an exhaustive-input
search budget. The contribution is best understood as a new \emph{adequacy
criterion}---output-interaction coverage---together with a constructive method to
satisfy it, unifying behavioural testing of classical and quantum systems.

\paragraph{Future work}
Three directions stand out: (i) learning or optimising the output partition
$\{\pi_j\}$ automatically, so that abstraction is data-driven rather than
hand-specified; (ii) validating that $\OCov_s$ predicts real defect discovery on
production ML pipelines and on quantum hardware, beyond seeded faults; and (iii)
tightening the covering-array construction with constraint solvers to exploit
output-side feasibility structure, further shrinking suites at higher strengths.

\section*{Data Availability}
The complete implementation, all benchmarks and baselines, the experiment drivers,
and scripts that regenerate every table and figure from raw result artifacts are
available with pinned dependencies and a container image for deterministic
reproduction,
\ifanon
at an anonymised artifact repository provided to reviewers.
\else
at \url{https://github.com/laminerihani-lab/reverse-nwise}.
\fi

\bibliographystyle{IEEEtran}
\bibliography{references}

\end{document}